%% file: arxiv_main.tex
\documentclass{article} 

\input{math_commands.tex}

\usepackage{hyperref}
\usepackage{url}
\usepackage{enumitem,kantlipsum}

\usepackage[margin=1in]{geometry}

\usepackage{amsthm,amsfonts,amsmath,amssymb,epsfig,color,float,graphicx,verbatim}
\usepackage[ruled,vlined]{algorithm2e}
\usepackage[authoryear]{natbib}
\usepackage{wrapfig}
\usepackage{authblk}

\usepackage{hyperref}
\hypersetup{
	colorlinks   = true, 
	urlcolor     = blue, 
	linkcolor    = blue, 
	citecolor   = black 
}
\usepackage{bbm}
\usepackage{newtxtext}
\usepackage[varbb]{newtxmath}
\usepackage{caption}
\usepackage{subcaption}
\usepackage{enumitem}
\usepackage{multicol}

\newtheorem{theorem}{Theorem}[section]

\newtheorem{lemma}[theorem]{Lemma}
\newtheorem{corollary}[theorem]{Corollary}
\newtheorem{definition}[theorem]{Definition}
\newtheorem{remark}[theorem]{Remark}

\newcommand{\reals}{\mathbb{R}}

\usepackage{algorithmic}

\newcommand{\tr}{\text{Tr}}

\newcommand{\spec}{\mathrm{sp}}
\newcommand{\spn}{\mathrm{span}}

\newcommand{\val}{\mathrm{Val}}
\newcommand{\fr}{\mathrm{fr}}

\newcommand{\be}{\mathbf{e}}
\newcommand{\bx}{\mathbf{x}}
\newcommand{\bw}{\mathbf{w}}

\newcommand{\bv}{\mathbf{v}}

\newcommand{\by}{\mathbf{y}}

\newcommand{\Xcal}{\mathcal{X}}

\newcommand{\Ical}{\mathcal{I}}
\newcommand{\Dcal}{\mathcal{D}}

\newcommand{\Hcal}{\mathcal{H}}
\newcommand{\Rcal}{\mathcal{R}}

\newcommand{\Ycal}{\mathcal{Y}}
\newcommand{\Pcal}{\mathcal{P}}

\newcommand{\norm}[1]{\|#1\|}
\newcommand{\inner}[1]{\left\langle#1\right\rangle}

\newcommand{\sd}{\mathrm{Rx}}
\renewcommand{\eqref}[1]{Eq.~(\ref{#1})}
\newcommand{\lemref}[1]{Lemma~\ref{#1}}
\newcommand{\thmref}[1]{Thm.~\ref{#1}}

\newcommand{\appref}[1]{Appendix~\ref{#1}}

\renewcommand{\secref}[1]{Section~\ref{#1}}
\renewcommand{\eqref}[1]{Eq.~(\ref{#1})}

\makeatletter
\newcommand{\printfnsymbol}[1]{%
  \textsuperscript{\@fnsymbol{#1}}%
}
\makeatother

\title{RedEx: Beyond Fixed Representation Methods \\via Convex Optimization}

\author[1, 2]{Amit Daniely}
\author[1]{Mariano Schain}
\author[1, 3]{Gilad Yehudai}

\affil[1]{\small Google Research Tel Aviv}
\affil[2]{\small Hebrew University of Jerusalem}
\affil[3]{\small Weizmann Institute of Science}

\date{}
\begin{document}
\maketitle

\begin{abstract}
Optimizing Neural networks is a difficult task which is still not well understood. On the other hand, fixed representation methods such as kernels and random features have provable optimization guarantees but inferior performance due to their inherent inability to learn the representations. In this paper, we aim at bridging this gap by presenting a novel architecture called RedEx (Reduced Expander Extractor) that is as expressive as neural networks and can also be trained in a layer-wise fashion via a convex program with semi-definite constraints and optimization guarantees. We also show that RedEx provably surpasses fixed representation methods, in the sense that it can efficiently learn a family of target functions which fixed representation methods cannot.
\end{abstract}

\section{Introduction}

Neural networks have demonstrated unparalleled performance in various tasks, including Computer Vision and Natural Language Processing (NLP). However, training them remains a challenging task that is not yet fully understood. On the theoretical side, the optimization landscape of neural networks is highly non-convex, characterized by numerous spurious local minima \citep{safran2018spurious,yun2018small} and are often also non-smooth. Consequently, proving optimization results for non-convex and non-smooth functions is generally deemed unfeasible \cite{kornowski2021oracle}. On the practical side, the optimization process for neural networks primarily employs gradient-based methods like Stochastic Gradient Descent (SGD) or ADAM \cite{kingma2014adam}, necessitating a meticulous search for hyperparameters. This process often relies on trial and error rather than being firmly grounded in theory.

On the contrary, fixed representation methods, such as kernels and random features, can be efficiently learned with provable guarantees using convex optimization techniques. However, recent research has highlighted a limitation: as they do not learn a representation, these methods are inherently less powerful than neural networks. There are learning scenarios where neural networks demonstrate efficient learning, while fixed representation methods falter (e.g. \cite{yehudai2019power, kamath2020approximate, malach2021connection, ghorbani2019limitations, daniely2020learning}). Different approaches that do facilitate efficient and provable representation learning often rely on overly simplistic models (e.g. \cite{yehudai2020learning, vardi2021learning, bietti2022learning}) or necessitate stringent assumptions about the data and employ specialized algorithms tailored to specific learning contexts (e.g. \cite{ge2017learning, allen2020backward, abbe2021staircase}).

A natural question that arises is whether there exists a model-class which benefits from the "best of all worlds", namely:

\begin{quote}
    \emph{Is there a model class that can be learned efficiently without making assumptions about the input distribution, matches the expressiveness of neural networks, and is capable of learning meaningful representations rather than relying on fixed ones?}
\end{quote}

In this paper, we provide an affirmative answer to this question by introducing the \textbf{Reduced Extractor Expander (RedEx)} architecture. We demonstrate that RedEx is as expressive as neural networks and can be learned using a convex program without any assumptions on the input data. Moreover, we establish that RedEx learns non-trivial representations, as evidenced by a novel learning problem we introduce that RedEx can efficiently learn, while fixed representation methods cannot.
In more details, our main contributions are:

\begin{enumerate}
    \item We introduce the RedEx architecture and show that it can efficiently express any Boolean circuit (\thmref{thm:RedEx epressive}). 
    \item We present an efficient polynomial-time algorithm for training RedEx, based on convex Semidefinite Programming (SDP) (\thmref{thm:reduction learning RedEx} and \algref{alg:1-layer RedEx}).
    \item We introduce a learning problem, based on a variation of the sparse-parity task which RedEx can learn efficiently, while any fixed-representation methods cannot (\thmref{thm:main_layerwise} and  \thmref{thm:kernel_lowerbound}).
    
\end{enumerate}

Furthermore, we demonstrate that if the output is one-dimensional, RedEx can be trained using standard gradient-based methods like gradient descent, without the need for SDP. Finally we extend the RedEx architecture to the convolutional setting.

\subsection{Related Works}
\paragraph{Fixed representation methods and NTK.} 
    Fixed representation methods are models which can be viewed as a feature mapping which is non-linear and fixed followed by a learned linear mapping. This includes kernel methods, random features \citep{RahimiRe07}, and others.
    In recent years, neural networks under certain assumptions were analyzed in the so called "kernel regime" \citep{woodworth2020kernel}. In this approach, it is assumed that the training takes place near the initial weights. This allows to analyse neural networks as if it is a fixed representation method.
    This approach was popularized through the Neural Tangent Kernel (NTK) model \citep{jacot2018neural}. Many similar works have shown positive results where neural networks can provably learn under different assumptions, e.g. \cite{daniely2019neural, andoni2014learning, du2017gradient, daniely2017sgd, allen2019learning, li2018learning, cao2019generalization}.


\paragraph{Limitations of fixed representation methods.} Several works in recent years have focused on the limitations of fixed representation methods, NTK and learning under the "kernel regime". \cite{yehudai2019power} and \cite{kamath2020approximate} have shown that fixed representation methods cannot learn even a single ReLU neuron under Gaussian distribution, unless the number of features is exponential in the input dimension. On the other hand, neural networks were shown to be able to efficiently learn single neurons \citep{yehudai2020learning,vardi2021learning}.  Several other works have shown that under certain distributional assumptions fixed representation methods cannot learn parity functions while neural networks can (see e.g. \cite{malach2021connection, daniely2020learning, malach2021quantifying}). Finally, \cite{ghorbani2019limitations,ghorbani2021linearized} have shown that the NTK and random features methods can essentially learn efficiently only low degree polynomials.

\paragraph{Provable optimization beyond fixed representations.} Several works consider model-classes which go beyond fixed representations, but can be efficiently and provably learned. These works usually consider either overly-simplistic models, or have strong assumptions on the input data. \cite{yehudai2020learning,vardi2021learning,bietti2022learning,bruna2023single,frei2020agnostic} consider learning single neurons or single index neurons with provable optimization guarantees. However these models are overly-simplistic and have very limited expressiveness. \cite{abbe2021staircase,allen2020backward} consider a certain hierarchical model resembling RedEx and show a family of functions that these models can learn. However the guarantees are for a specific family of input distributions, with a training algorithm that is tailored for these specific learning problems. \cite{ge2017learning,tian2017analytical} consider learning a one-hidden layer neural network with gradient descent for Gaussian inputs using a specific analytic formula relying on the distribution of the data. Our model also share similarity to phase retrieval methods (e.g. \cite{candes2013phaselift,
candes2015phase}), although these works mostly consider Gaussian data or data distributed uniformly on a sphere. 

\cite{livni2014computational} consider the problem of learning a one-hidden layer network with square activation under trace norm constraints. They prove learnability using a reduction to a convex program relying on the GECO algorithm \citep{shalev2011large}. Our work is similar to that in nature, however we provide several extensions: (1) A separation result between fixed representation methods and RedEx, which do not appear in \cite{livni2014computational}; (2) Extensions to multivariate output and a convolutional structure, which is not possible using the convex reduction in \cite{livni2014computational}; and (3) A multilayer version of RedEx which enables to express any Boolean circuit, and thus match the expressive power of neural networks.

\section{Notations and Settings}
We denote vectors in bold-face: $\bx$. We will assume the input space is $\reals^d$. The output 
space will be denoted $\Ycal$. We will consider algorithms 
that learn functions from $\reals^d$ to $\reals^k$ and are 
evaluated by a convex loss function $\ell:\reals^k\times \Ycal\to [0,\infty)$. Given a distribution $\Dcal$ on 
$\reals^d\times\Ycal$ and $h:\reals^d\to\reals^k$
we denote 
$\ell_\Dcal(h) = \E_{(\bx,y)\sim\Dcal}\ell(h(\bx),y)$. Likewise, for a dataset $S=\{(\bx_1,y_1),\ldots,(\bx_m,y_m)\}$ we denote $\ell_S(h)=\frac{1}{m}\sum_{i=1}^m \ell(h(\bx_i),y_i)$.

For $\bx\in\reals^d$ we denote by $\bx^{\otimes 2}:=\bx\bx^\top\in \reals^{d\times d}$ the outer product. 
For a vector, $\bx\in\reals^d$ we will use $\|\bx\| = \sqrt{\sum_{i}x^2_i}$ to denote the Euclidean norm. Given $I\subset [d]$ we denote $\bx(I) = \sum_{i\in I}x_i$.
For a matrix $A$, we will use $\|A\|_\fr = \sqrt{\sum_{ij}A^2_{ij}}$ to denote the Frobenius norm,  $\|A\|_\spec = \max_{\|\bx\|=1}\|A\bx\|$ to denote the spectral norm, and $\|A\|_\tr$ to denote the trace norm which is the sum of $A$'s singular values.  For a diagonal matrix $D$ we denote by $|D|$ the  diagonal matrix whose $i$-th diagonal coordinate is equal to $|D_{i,i}|$. If $D$ is also Positive Semi-Definite (PSD) we define by $\sqrt{D}$ the matrix whose $i$-th diagonal coordinate is equal to $\sqrt{D_{i,i}}$.  
We will use $\vec{A}$ to denote a tuple of matrices $\vec{A} = (A_1,\ldots,A_k)\in \left(\reals^{d\times n}\right)^k$. We will let $\|\vec{A}\|^2_\fr=\sum_{i=1}^k\|A_i\|^2_\fr$.
For $B\in \reals^{m\times d}$ and $C\in \reals^{n\times m}$ we denote $B\vec{A}= (BA_1,\ldots,BA_k)$ and $\vec{A}C= (A_1C,\ldots,A_kC)$.
We denote by $B^d_M\subset\reals^d$ the Euclidean ball of radius $M$ centered at $0$.

For a symmetric matrix $A\in \reals^{d\times d}$, we say that $A = U^\top D U$ is a {\em compact orthogonal diagonalization} if $U\in \reals^{d'\times d}$ is a matrix with $d'\le d$ orthonormal rows and $D\in \reals^{d'\times d'}$ is a diagonal matrix with non-zero diagonal entries, or the $0$ matrix in $\reals^{1\times 1}$. Note that any symmetric matrix has a compact orthogonal diagonalization. For a linear subspace $V\subset \reals^m$, we denote by $P_V$ the projection on $V$ and say that $A$ is $V$-supported if $A = P_V^\top A P_V$.

For an embedding $\Psi:\Xcal\to\reals^n$ we denote by $\Hcal_{\Psi}$ the space of all function for which there is $\bv\in\reals^n$ such that $\forall\bx\in \Xcal,\;h(\bx) = \inner{\bv,\Psi(\bx)}$. We also define a norm on $\Hcal_\Psi$ by $\|h\|_{\Psi} = \min\{\|\bv\| : \forall\bx\in \Xcal,\;h(\bx) = \inner{\bv,\Psi(\bx)}\}$. We note that $\|\cdot\|_{\Psi}$ turn $\Hcal_\Psi$ into a Hilbert space. We also define $k_{\Psi}(\bx,\by) = \inner{\Psi(\bx),\Psi(\by)}$. For two matrices $A,B\in\reals^{n\times m}$ we define $\inner{A,B}:= \tr(B^\top A)$.

\section{Reduced Extractor-Expanders (RedEx)}\label{Sec:RedEx}

In this section we present the main architecture that we analyse throughout the paper. This architecture is aimed at being analogous to 2-layer neural networks with a quadratic activation. Its base component consists of two layers: The first layer "extracts" the most informative directions of the data using a matrix with orthogonal rows, and bounded Frobenius norm, and then ``expands" it quadratically by increasing the dimension. The second layer is a linear transformation with bounded norm over the extracted features.

\begin{definition}[RedEx - Reduced Extractor-Expander]\label{def:redex}
    A function $\Lambda:\reals^d\rightarrow\reals^k$ is called a \textbf{RedEx} (Reduced Expander Extractor) of width $M$ if it is of the form: $\Lambda = \Psi_{\vec{P}} \circ \Psi_V$ where:
    \begin{enumerate}
        \item The function $\Psi_V:\reals^d\rightarrow \reals^{d'\times d'}$ is of the form: 
        \[
            \Psi_V(\bx) = (V\bx)^{\otimes 2}
        \]
        for $V\in \reals^{d'\times d}$ with $d'\le d$, orthogonal rows, and $\|V\|^2_\fr \le M$. We call the matrix $V$ {\em extractor}. 
        \item The function $\Psi_{\vec{P}}:\reals^{d'\times d'}\rightarrow\reals^k$ is of the form:
        \[
            \Psi_{\vec{P}}(X) = \left( \inner{P_1,X},\dots,\inner{P_k,X}\right)
        \]
        for $P_i\in\reals^{d'\times d'}$ with $\norm{P_i}_\text{sp}\leq 1$.
    \end{enumerate} 
    The function $\Psi_V$ is called an \textbf{extractor-expander}.
\end{definition}

To have some intuition, the width of an extractor $V$
can be thought of as a continuous surrogate to the number of orthogonal directions $V$ extracts. Indeed, in order to extract $d'$ orthogonal dimensions defined by unit vectors $\be_1,\ldots,\be_{d'}$ we can use the extractor $V\in \reals^{d'\times d}$ whose $i$'th row is $\be_i^\top$. In this case, the width of $V$ is $d'$. Definition \ref{def:redex} generalizes such extractor matrices, and allows to give larger weights to directions which are ``more important".
The width $M$ controls the expressivity of the architecture. 
Allowing large width will result with more functions that can be expressed, but on the other hand will require more examples to learn them. Alternatively, small width will result with a less expressive class of function, but with better generalization capabilities.

Given a data set $(\bx_1,\by_1),\dots (\bx_m,\by_m)\in \reals^d\times \reals^k$ our algorithm will seek a RedEx $\Psi_{\vec{P}} \circ \Psi_V$
that minimizes the loss subject to a width constraint. We will also allow for an additional small regularization term, that will be used to guarantee generalization. 
Specifically, given a loss function $\ell$, and for the regularization function
\begin{equation}\label{eq:reg_function_intro}
\Rcal(\vec{P},V) = \|V^\top V\|^2_\fr + \sum_{i=1}^k\|V^\top P_iV\|^2_\fr    
\end{equation}
our algorithm will minimize: $\ell_S(\Psi_{\vec{P}} \circ \Psi_V) + \lambda \Rcal(\vec{P},V)$, subject to the constraint that the width of $V$ is at most $M$.
We will later explain how this can be done in polynomial time, and will provide guaranties on its performance.
We first extend the above architecture to multi-layer RedEx in the same manner that a 2-layer neural network is extended to multi-layer. The basic idea is to use several extractor-expanders $\Psi_{V^t}$ in a sequential manner, and at the last layer use a linear transformation $\Psi_{\vec{P}}$.
As with the basic depth-two RedEx architecture, the width of the extractors $V^t$ will control the complexity of the functions computed by the architecture, and will be used to trade-off expressive power and sample complexity.

There are two issues that arise when doing such a generalization: (1) The representation dimension grows exponentially with the number of layers. This is because the function $\Psi_V(\bx) = (V\bx)^{\otimes 2}$ expands the dimension quadratically. (2) Training all the layers simultaneously is computationally hard (as we show later, poly-sized deep RedEx architectures can express any poly-sized Boolean circuit, which implies that they are hard to learn \cite{kearns1994cryptographic}), whereas one of our main goals is to obtain provable guarantees for the optimization process. 

To deal with the first issue,
we allow extractors whose output dimension is at most the number of examples $m$. This limits the representation dimension to $m^2$. This is a convenient way to deal with this issue theoretically, however in practice alternative approaches might be favourable. 
For instance, we can simply delete all the rows in the extractors $V$ whose norm is $\le \epsilon$ for some tunable parameter $\epsilon$. For a sufficiently small $\epsilon$, this will not alter the solution by much. We note that the number of rows with norm that is larger than $\epsilon$ is at most the width of the extractor, divided by $\epsilon^2$. 
It is also possible to use the kernel trick. Lastly, a more practically oriented way, is to apply a dimension reduction method, such as PCA, after every extractor-expander layer.

To address the second issue, instead of training all the layers simultaneously, we train them sequentially. This is analogous to layer-wise training in neural networks \cite{bengio2006greedy}. The idea is that for $t$-th layer, we find functions $\Psi_{V^{t}}$ and $\Psi_{\vec{P}^{t}}$ which minimizes the target loss. After training is finished we only keep the representation function $\Psi_{V^{t}}$, while discarding the linear transformations. The input for the next layer $t+1$ is the output of the extractor-expander $\Psi_{V^{t}}$.

Another small issue that arises for extractor-expanders is that they only allow to compute degree-2 polynomials which are even (i.e. satisfy $p(\bx)=p(-\bx)$). More generally, multi-layer RedEx would only allow to compute even high degree polynomials. This issue can be easily fixed by adding an extra fixed coordinate to the inputs. We summarize all the above in Algorithm \ref{alg:train RedEx}.

\begin{algorithm}[ht]
	\caption{Training multi-layer RedEx}
	\begin{algorithmic}[1]\label{alg:train RedEx}
		\STATE \textbf{Parameters: } A loss $\ell:\reals^k\times \Ycal\to [0,\infty)$, number of layers $L$, width parameters $M_1,\ldots,M_L$,  regularization parameters $\lambda_1,\ldots,\lambda_L$, and constant parametr $c$.
		\STATE \textbf{Input: } A dataset $(\bx_1,y_1),\dots (\bx_m,y_m)\in \reals^d\times \Ycal$.
		\STATE Define $\bx_1^{0} = \begin{pmatrix}c \\ \bx_1 \end{pmatrix},\dots, \bx_m^{0} = \begin{pmatrix}c \\ \bx_m \end{pmatrix}$ and $d_0:= d + 1$ 
		\FOR{t=1,\dots,L}
		\STATE Find $V^{t}$ and $\vec{P}^{t} = \left(P_1^{t},\dots, P_k^{t}\right)$ that minimizes:
		\begin{equation}\label{eq:min inside RedEx algo}
		\frac{1}{m}\sum_{i=1}^{m} \ell\left(\Psi_{\vec{P}^{t}} \circ \Psi_{V^{t}}(\bx_i^{t-1}), y_i\right)	+ \lambda_t\Rcal(\vec{P}^{t}, V^{t})
		\end{equation}
		where: 
		\begin{enumerate}
		    \item $\Psi_{\vec{P}^{t}} \circ \Psi_{V^{t}}(\bx) := \left(\inner{P_1^{t}, (V^{t}\bx)^{\otimes 2}},\dots,\inner{P_k^{t}, (V^{t}\bx)^{\otimes 2}}\right)$.
		    \item $V^{t}$ has $n_t\le m$ orthogonal rows. Each row is a vector in $\reals^{d_{t-1}}$, and $\|V^t\|_F^2\le M_t$
		    \item For each $j\in[k]$, $P_j\in\reals^{n_t\times n_t}$ and $\norm{P^{t}_j}_\spec\leq 1$.
		\end{enumerate}
		\STATE Define $d_t = n^2_t$ and $\bx_i^{t} = \Psi_{V^{t}}\left(\bx_i^{t-1}\right)$ for $i\in[m]$. We will view $\bx_i^{t}$ as a vector in $\reals^{d_t}$ 
		\ENDFOR
		\STATE \textbf{Output:} Output the hypothesis $\Psi_{\vec{P}^{L}}\circ \Psi_{V^{L}} \circ \Psi_{V^{L-1}} \circ\ldots \circ \Psi_{V^{1}} $
\end{algorithmic}
\end{algorithm}

The heart of the above algorithm is to minimize \eqref{eq:min inside RedEx algo} in step 5. On one hand, it can be done using standard gradient methods such as GD or SGD. The problem with this approach is that the function being optimized is not convex, even without the norm and orthogonality constraints. Thus, it is not clear that it converges to a global optimum. In the next section we provide an efficient algorithm for finding this optimum, thus providing an efficient and provable algorithm for layer-wise learning of multi-layer RedEx.

We emphasize that one caveat of the layer-wise optimization approach is that a global minimizer for all the layers simultaneously might achieve better performance than a global minimizer for each layer separately. It can be viewed as a greedy algorithm, where at each step we optimize the current layer which is locally the best possible step, but it may not be the best step globally (i.e. for all layers simultaneously). The main advantage of this sequential approach is that it will allow us to use convex optimization and give provable guarantees for optimizing our model, while also providing separation between this model and fixed representation methods. Note that there are no layer-wise training guarantees for neural networks in general. 

We now show an expressivity result, namely that the RedEx architecture can approximate any Boolean circuit with only a quadratic increase of the size of the circuit. This shows that like neural networks, our architecture can express a very large class of functions -- virtually any function of interest. Indeed, any function that can be computed efficiently has a small circuit that computes it \cite{vollmer1999introduction}.
\begin{theorem}\label{thm:RedEx epressive}
    Let $B:\{0,1\}^d\rightarrow\{0,1\}$ be a function computed by a Boolean circuit of size $T$. Then we can define a RedEx with depth $O(T)$ and intermediate feature dimension at most $O(T^2)$ that computes $B$.
\end{theorem}
The proof can be found in \appref{appen:proofs from sec RedEx}. \thmref{thm:RedEx epressive} implies that the RedEx architecture is as expressive as neural networks. This can be seen using the following simple argument:
Any neural network (with inputs in $\{0,1\}^d$ and output in $\{0,1\}$) can be simulated by a boolean circuit, where the number of nodes in the circuit is at most polynomial in the number of parameters (see \cite{maass1997networks}). \thmref{thm:RedEx epressive} shows that any Boolean circuit can be simulated by a multilayer RedEx architecture with at most polynomial blow-up in the size of the circuit. Thus, given a neural network, it can be simulated by a multilayer RedEx architecture with at most polynomial blow-up in the number of parameters.

\section{Efficient and Provable Learnability of RedEx}\label{sec:learnability of RedEx}
In this section we present an efficient algorithm for learning a single layer RedEx, and present generalization guaranties for it. This algorithm can be used to minimize objective \eqref{eq:min inside RedEx algo}, thus leading to an efficient implementation of algorithm \ref{alg:train RedEx}.

Our approach is to reduce the problem of minimizing \eqref{eq:min inside RedEx algo} under the norm and orthogonality constraints to a convex semi-definite program, which contains PSD constraints. Such a problem can be  solved using convex SDP algorithms in polynomial time.

In order to do so, we present a different parametrization of RedEx functions. Let 
\[
\Pcal_M = \{(\vec{P},V) : \|V\|^2_\fr\le M\text{ and }\forall i,\|P_i\|_\spec\le 1\}
\]
A RedEx function of width $M$ is defined by $(\vec{P},V)\in\Pcal_M$. Now, let
\[
\Pcal^M = \{(\vec{A},R) : \tr(R)\le M\text{ and }\forall i,-R\preceq A_i \preceq R\}
\]
As the following lemma shows, we can alternatively define  width $M$ RedEx functions via $(\vec{A},R)\in\Pcal^M$.
We then show that under this alternative parameterization of RedEx functions, objective \eqref{eq:min inside RedEx algo} becomes convex. Furthermore, we can efficiently convert the alternative parameterization to the original. These two facts enable us to efficiently implement algorithm \ref{alg:1-layer RedEx}.

\begin{theorem}\label{thm:reduction learning RedEx}
    Let $\mathcal{H}_M$ be the class of functions of the form:
    \[
        h_{V,\vec{P}}(\bx) = \left(\inner{P_1, (V\bx)^{\otimes 2}},\dots,\inner{P_k, (V\bx)^{\otimes 2}}\right)
    \]
    For $(\vec{P},V)\in\Pcal_M$ . Let $\mathcal{H}^M$ be the class of functions of the form:
        \[h^{R,\vec{A}}(\bx) =             \left(\inner{\bx,A_1\bx},\dots,\inner{\bx,A_k\bx}\right)
    \]
    For $(\vec{A},R)\in\Pcal^M$. We have:
    \begin{enumerate}
        \item $\mathcal{H}_M = \mathcal{H}^M$.
        \item Fix $(\vec{A},R)\in\Pcal^M$. Diagonalize $R = U^\top D U$ for unitary $U$ and diagonal $D$ and let 
        \[
        V = \sqrt{D}U,~~ P_i = (V^\dagger)^\top A_i (V^\dagger)
        \]
        then $(\vec{P},V)\in\Pcal_M$ and $h_{V,\vec{P}} = h^{R,\vec{A}}$. Furthermore, $\tr(R) = \|V\|^2_{\fr}$ and $\Rcal(\vec{P},V) = \|R\|^2_\fr + \|\vec{A}\|^2_\fr$
    \end{enumerate}
\end{theorem}

The proof can be found in \appref{appen:proofs from learnability of RedEx}. Theorem \ref{thm:reduction learning RedEx} suggests the following algorithm for training a single RedEx layer, and to optimize objective \eqref{eq:min inside RedEx algo} in algorithm \ref{alg:train RedEx}.

\begin{algorithm}[H]
	\caption{Training 1 -layer RedEx}
	\begin{algorithmic}[1]\label{alg:1-layer RedEx}
		\STATE \textbf{Parameters: } Loss $\ell:\reals^k\times \Ycal\to [0,\infty)$, width parameter $M$ and regularization parameter $\lambda$
		\STATE \textbf{Input: } A dataset $(\bx_1,y_1),\ldots,(\bx_m,y_m) \in \reals^d\times \Ycal$
		\STATE Find symmetric $d\times d$ matrices $A_1,\ldots,A_k, R$ by solving the semi-definite program:
\begin{eqnarray} \label{program:basic_2}
\min && \frac{1}{m}\sum_{i=1}^m \ell_{y_i}\left(\bx_i^\top A_1\bx_i,\ldots,\bx_i^\top A_k\bx_i\right) 
+  \lambda(\|R\|_\fr^2+\|\vec{A}\|_\fr^2 )
\\
s.t. &&  -R \preceq A_i \preceq R  \nonumber
\\
&& R \succeq 0 \nonumber
\\
&& \tr(R) \le M
\end{eqnarray}
\STATE Compute an  orthogonal diagonalization $R = U^\top DU$ 
\STATE Output $V = \sqrt{D}U$ and $P_i = (V^{\dagger})^\top A_i V^{\dagger}$.  
\end{algorithmic}
\end{algorithm}

\begin{remark}
We note that \algref{alg:1-layer RedEx} can be performed in polynomial time, although we don't specify the exact training time for the algorithm. The reason is that there are different convex program solvers with different pros and cons, and the training time depends on which solver is chosen. The minimization objective in \eqref{program:basic_2} is a general strongly convex function under semi-definite constraints. It can be solved using general interior point method \citep{potra2000interior}, the ellipsoid algorithm (see e.g. Ch.2 in \cite{bubeck2015convex}), conic optimization (e.g. \cite{auslender2006interior, dahl2022primal}) or any other method which solves convex SDP problems.


\end{remark}

\begin{remark}\label{remark:restrict_extractor_dim}
    We have that $\mathrm{rank}(R)\le m$, and hence $V$ has 
    at most $m$ non-zero rows, that is the number of "improtant features" is bounded by the size of the dataset. Indeed, let $P:\reals^d\to \reals^d$
    be the projection on $\spn\{\bx_1,\ldots,\bx_m\}$, and let $(\vec{A},R)$ be an optimal solution  to program \eqref{program:basic_2}. Note that there is a single optimal solution, as program \eqref{program:basic_2} is strongly convex. It is not hard to see that the objective value of $(P^\top\vec{A}P,P^\top R P)$ is as good as the objective value of $(\vec{A},R)$, this is because projection on the data samples produces the same class of functions, while it does not increase both the trace and Frobenius norms. As the optimal solution is unique, we conclude that  $R=P^\top R P$.
\end{remark}
We next state a generalization result for algorithm \ref{alg:1-layer RedEx}. The result follows directly from Corollary 13.6 in \cite{shalev2014understanding}, by noticing that the objective is convex with an appropriate regularization term. To this end, we define
\begin{equation}\label{eq:val_def}
\val_{\Dcal}(\vec{A},R) :=    \E_{(\bx,y)\sim\Dcal} \ell_{y}\left(\bx^\top \vec{A}\bx\right) 
\end{equation}
and
\[
\val_{\Dcal,M} := \inf_{-R\preceq A_i\preceq R\text{ and }\tr(R)\le M}\val_{\Dcal}(\vec{A},R)
\]
\begin{theorem}\label{thm:generalization}
Assume that the dataset is an i.i.d. sample from a distribution $\Dcal$ on $B^d_{M_1}\times\Ycal$ and that the loss is $L$-Lipschitz. Let $(\vec{A},R)$ be the output of algorithm \ref{alg:1-layer RedEx}. Then
\[
\E_S\val_{\Dcal}(\vec{A},R) \le  \val_{\Dcal,M}+ \lambda (k+1)M^2+\frac{M_1^4L^2}{\lambda m}
\]
\end{theorem}

\thmref{thm:generalization} shows a trade-off in generalization capabilities by choosing the parameter $M$, similar to the well known bias-variance trade-off. Namely, larger value of $M$ allows for better expressive power but requires more samples to achieve good generalization capabilities. 
We note that our generalization result scales at a rate of $O\left(\frac{1}{\sqrt{m}}\right)$ by choosing an appropriate $\lambda$. It is an interesting question whether this rate can be improved to $O\left(\frac{1}{m}\right)$, similarly to what is done in \cite{wang2021harmless}, but for non-smooth regularizers.

\section{Layerwise RedEx surpasses Kernel Methods}
In this section we will provide a learning problem which demonstrates a separation between RedEx and fixed representation methods. The problem we choose is inspired by \cite{daniely2020learning} where they show that neural networks can learn the sparse parity function under a certain distribution which "leaks" the coordinates of the parity. In more details, given an input space $\{\pm 1\}^d$, the sparse parity function on the $k$ coordinates $\bx_{i_1},\dots, \bx_{i_k}$ is defined as: $\prod_{j=1}^k\bx_{i_j}$. Since RedEx is learned in a layer-wise fashion, we consider a slightly different learning problem which better aligns with the RedEx architecture and still cannot be learned by fixed representation methods.

Namely, we consider the problem of learning the following family of models:
The input space is $\{\pm 1\}^d$, the output space is $\reals^{(1+k/2)}$ for even $k= \Theta\left(\frac{\sqrt{\log(d)}}{(\log\log(d))^2}\right)$, and the input distribution is uniform on $\{\pm 1\}^d$. Denote by $p_0,\ldots,p_{k}$, the set of orthogonal polynomials w.r.t. the distribution of $\sum_{i=1}^k X_i$ for i.i.d. Radamacher r.v. $X_1,\ldots,X_k\in \{\pm 1\}$. These polynomials are called {\em Kravchuk Polynomials} \cite{nikiforov1991classical} and are given by the recursion formula
\begin{equation}\label{eq:Kravchuk}
p_{0}(x)=1 ,\;\;\; p_{1}(x) = \frac{x}{\sqrt{k}},\;\; xp_{i}(x) = \sqrt{(i+1)(k-i)}p_{i+1}(x) + \sqrt{i(k-i+1)}p_{i-1}(x)
\end{equation}
We consider the problem of learning a function of the form 
\[
    h_\Ical(\bx) = \left(p_0\left(\sum_{i\in \Ical}x_i\right),p_2\left(\sum_{i\in \Ical}x_i\right),p_4\left(\sum_{i\in \Ical}x_i\right),\ldots,p_k\left(\sum_{i\in \Ical}x_i\right) \right)
\]
for an unknown set of coordinates $\Ical\subset [d]$ with $|\Ical| = k$ and w.r.t. the square loss $\ell(\hat \by,\by) = \|\hat \by-\by\|^2$.
Our first result shows that algorithm \ref{alg:train RedEx} learns a function with loss of $o(1)$. Note that the coordinates of $h_I$ are polynomials of increasing degree, while its last coordinate is the sparse parity function $\bx \mapsto \prod_{i\in \Ical}x_i$. To see that this is indeed the sparse parity function, note that by definition it is orthogonal to any Kravchuk polynomial of degree $i \leq k$, and the $k$-th Kravchuk polynomial is the unique polynomial with this property, hence it must be the sparse parity function.
Thus, our function can be seen as learning the parity function, but using a kind of "staircase property" \citep{abbe2021staircase} which the RedEx architecture exploits due to its layer-wise training.

\begin{theorem}\label{thm:main_layerwise}
Assume we run algorithm \ref{alg:train RedEx} on $m$ i.i.d. examples, $L=\lceil\log_2(k)\rceil$ layers, regularization parameters $\lambda_1=\lambda_2=\ldots=\lambda_L=\frac{1}{\sqrt{m}}$, width parameters $M_1 = \frac{1}{2}+\frac{1}{\sqrt{2-2/k}},\;M_2=\ldots=M_L = M = 2^{3k}(3k)^{k}$, and constant parameter $c=\sqrt{2}$.
Assume furthermore that each layer is trained using a fresh sample.
Then, w.p. $1-\delta$, for the output hypothesis $h$,  $\ell_{\Dcal}(h) = \frac{d^42^{O(k^2\log^2(k))}}{\delta m^{1/4}} = \frac{O(d^5)}{\delta m^{1/4}}$ 
\end{theorem}

The reason for sampling a batch of fresh samples when training each layer is a technical artifact of the proof, aimed at eliminating the dependence between the training of each layer. It can be seen alternatively as if the original dataset is larger by a factor of $\log(k):= O(\log\log(d))$, and we only use a part of it for training each layer.

We compliment the above result by showing that polynomial-time fixed-representation methods, such as kernels and random features, cannot achieve the guarantee in Theorem \ref{thm:main_layerwise}. The reason is that the last coordinate of $h_{\Ical}$ is the parity function $\bx \mapsto \prod_{i\in \Ical}x_i$. This implies that any fixed-representation method that is guaranteed to find a 
function $h$ with $\E_{\bx} \left\|h_{\Ical}(\bx) -h(\bx)\right\|^2_2 = o(1)$ has super-polynomial complexity of $d^{\Omega(k)}$. Specifically, Corollary 13 from \cite{ben2002limitations} implies:

\begin{theorem}\label{thm:kernel_lowerbound}
Let $\Psi:\{\pm 1\}^d\to B_{M_1}^m$ be any, possibly random, embedding. Assume that for any $\Ical\subset [d]$ with $|\Ical|=k$, w.p. $\ge 1/2$ over the choice of $\Psi$, there are vectors $\bw_0,\ldots,\bw_k\in B_{M_2}^m$ such that
\[
\E_\bx\sum_{j=0}^{k/2}\left(p_{2j}\left(\sum_{i\in \Ical}x_i\right) - \inner{\bw_j,\Psi(\bx)}\right)^2 \le 0.99
\]
Then $M_1M_2\ge d^{\Omega(k)}$
\end{theorem}

\subsection{On the proof of theorem \ref{thm:main_layerwise}}\label{sec:sketch_of_main_thm_proof}
Theorem \ref{thm:main_layerwise} is proved in \appref{sec:main_thm_proof}. In section \ref{sec:first_layer} it is shown that $V^1$, the first layer's extractor, "reveal" the important coordinates, in the sense that $V^1 \approx V^1P_\Ical$, where $P_\Ical$ is the projection on the coordinates in $\Ical$. It is also shown that the representation $\Psi_1$ computed by the first layer, is expressive enough so that $p_0\left(\sum_{i\in \Ical}x_i\right)\text{ and } p_2\left(\sum_{i\in \Ical}x_i\right)$ can be well approximated by  functions  $\bx\mapsto \inner{\bv_0,\Psi_1(\bx)}$ and  $\bx\mapsto \inner{\bv_2,\Psi_1(\bx)}$ for vectors $\bv_0,\bv_2$ with a norm bound that do not depend on $d$, but only on $|\Ical|$.
In section \ref{sec:remaining_layer} it is then shown by induction that $\Psi_t$ is expressive enough so that $p_0\left(\sum_{i\in \Ical}x_i\right),\ldots,\text{ and } p_{2^t}\left(\sum_{i\in \Ical}x_i\right)$ can be well approximated by  functions  $\bx\mapsto \inner{\bv_i,\Psi_t(\bx)}$ for vectors $\bv_i$ with a norm bound that depend only on $|\Ical|$. The reason is that each $p_i$ can be represented as a quadratic polynomial of $p_j$ for $j\le \lceil\frac{i}{2}\rceil$, with bounded coefficents, together with the fact that by the induction hypothesis these $p_j$'s can be expressed as a linear function on top of $\Psi_{t-1}$

\section{Extensions and Discussion}\label{sec:discussion}
In the following section we will show two extensions of the RedEx architecture -- using a norm formulation of the objective and extension to a convolutional structure. Importantly, the norm formulation of RedEx for a one-dimensional output can be trained without the semi-definite constraints, and thus be trained using standard gradient descent or any other non-constrained convex optimization methods.

\subsection{Norm Formulation of RedEx and Relation to Trace norm}

In Algorithm \ref{alg:1-layer RedEx} we gave a constraint optimization problem which can be solved using SDPs. In this section we show how to reformulate this problem as an unconstrained optimization via a new norm we define:


\begin{definition}\label{def:redex norm}
    For $\vec{A}_{1:k} = (A_1,\dots,A_k)\in \reals^{d\times d}$ where each $A_i$ is symmetric we define the \textbf{RedEx norm} as:
    \begin{align}\label{eq:sd norm def}
        \norm{\vec{A}_{1:k}}_{\text{Rx}}:= \min\left\{\tr(R): R\succeq 0,\text{ and } \forall i,~ -R \preceq A_i \preceq R \right\}
    \end{align}
\end{definition}
We first show that the above defined norm satisfies several properties:

\begin{lemma}{\textbf{Properties of $\| \cdot \|_{\text{Rx}}$}}\label{lem:properties of sd norm}
\begin{enumerate}
\item
$\| \cdot\|_\sd$ is a norm on $k$-tuples of symmetric matrices.
\item If $k=1$, then $\norm{\cdot}_{\text{Rx}}$ is equivalent to the trace norm. Additionally, if we write $A=U^\top D U$ for an orthogonal $U$ and diagonal $D$, then $R = U^\top |D| U$.
\end{enumerate}
\end{lemma}
The proof can be found in \appref{appen:proofs from discussion}. Item (3) gives a very simple expression for the RedEx norm in the case for $k=1$, however we are not aware of a simple expression for $\norm{\cdot}_\sd$ where $k\geq 2$. We can now optimize \eqref{eq:min inside RedEx algo} using the RedEx norm. For that we replace the minimization problem in Algorithm \ref{alg:1-layer RedEx} by:

 \begin{eqnarray} \label{program:basic_2_with_norm}
\min && \frac{1}{m}\sum_{i=1}^m \ell_{y_i}\left(\bx_i^\top A_1\bx_i,\dots,\bx_i^\top A_k\bx_i\right) 
+ \lambda_1\|\vec{A}\|_\sd + \lambda_2\|\vec{A}\|_\fr^2
\end{eqnarray}
We give the full algorithm in \appref{appen:proofs from discussion}. Note that we don't need to minimize over the Frobenius norm of $R$, as it is already done by minimizing the RedEx norm. The caveat of \eqref{program:basic_2_with_norm} is that we currently don't know how to calculate the gradient of the RedEx norm directly (i.e. without calculating $R$), or the projection on norm-induced balls unless we resort to general convex SDP solvers. Hence, at the moment we don't know how to utilize the norm formulation in order to design faster algorithms.

One major practical improvement on the training of RedEx that we can make is in the case where our goal is to learn a function $f:\reals^d\rightarrow \reals$ (i.e. the output dimension $k=1$). For this case, we can use the characterization in \lemref{lem:properties of sd norm}~(3), where for $k=1$ the RedEx norm is equivalent to the trace norm. In this case, we can replace the minimization problem in Algorithm \ref{alg:1-layer RedEx} by:

\begin{eqnarray*} \label{program:basic_2_one_dim}
\min && \frac{1}{m}\sum_{i=1}^m \ell_{y_i}\left(\bx_i^\top A\bx_i\right) 
+ \lambda_1\|A\|_\tr + 2\lambda_2\|A\|_\fr^2
\end{eqnarray*}
This problem is substantially easier than minimizing \eqref{program:basic_2}, since it is an unconstrained convex optimization problem that can be solved by standard GD or SGD. Note that to find $V$ we don't need to find $R$, since by \lemref{lem:properties of sd norm}~(3) we can compute a diagonalization $A = U^\top DU$, and then output $V = \sqrt{|D|}U$ and $P = (V^{\dagger})^\top A V^{\dagger}$. We give the full algorithm in \appref{appen:proofs from discussion}.

\subsection{Convolutions}\label{subsec:conv}
One of the advantages of neural networks is that it allows to choose an architecture according to the structure of the data. A central example is convolutional networks for data which is translation invariant such as images. The input vector for a convolutional layer is divided into patches. In other words, it is a vector $(\bx_1,\ldots,\bx_p)\in \left(\reals^d\right)^p$. A convolutional layer applies on each patch the same linear function followed by a non-linearity. That is, it computes a mapping of the form 
$(\bx_1,\ldots,\bx_p) \mapsto (\sigma(W\bx_1),\ldots,\sigma(W\bx_p))$

For a matrix $W\in\reals^{d'\times d}$ and some non-linearity $\sigma$. A convolutional extractor-expander works in a similar fashion. It applies the same extractor $V$ to all patches, and then expand each patch quadratically. This is detailed in the following definition.

\begin{definition}[Convolutional RedEx]\label{def:conv_redex}
    A function $\Lambda:\left(\reals^d\right)^p\rightarrow\reals^k$ is called a \textbf{Convolutional RedEx} of width $M$ if it is of the form: $\Lambda = \Psi_{\vec{P}} \circ \Psi_V$ where:
    \begin{enumerate}
        \item The function $\Psi_V:\left(\reals^d\right)^p\rightarrow \left(\reals^{d'\times d'}\right)^p$ is of the form: 
        \[
            \Psi_V(\bx_1,\ldots,\bx_p) = \left((V\bx_1)^{\otimes 2},\ldots,(V\bx_p)^{\otimes 2}\right)
        \]
        for $V\in \reals^{d'\times d}$ with $d'\le d$, orthogonal rows, and $\|V\|^2_\fr \le M$.
        \item The function $\Psi_{\vec{P}}:\left(\reals^{d'\times d'}\right)^p\rightarrow\reals^k$ is of the form:
        \[
            \Psi_{\vec{P}}(X_1,\ldots,X_p) = \left( \sum_{j=1}^p\inner{P_{1,j},X_j},\dots, \sum_{j=1}^p\inner{P_{k,j},X_j}\right)
        \]
        for $P_{i,j}\in\reals^{d'\times d'}$ with $\norm{P_{i,j}}\leq 1$.
    \end{enumerate}
    The function $\Psi_V$ is called a \textbf{convolutional extractor-expander}.
\end{definition}
As with the basic version of RedEx, we can extend the basic convolutional RedEx architecture to a multilayer architecture. Likewise, a single layer of convolutional RedEx s can be trained efficiently, similarly to a single layer of RedEx. A multilayer convolutional RedEx can be trained efficiently in a layerwise manner, as basic RedEx. We outline next the algorithm for learning a single convolutional RedEx layer. The extension to multilayer is straight forward.

\begin{algorithm}[H]
	\caption{Training 1 -layer convolutional RedEx}
	\begin{algorithmic}[1]\label{alg:1-layer conv RedEx}
		\STATE \textbf{Parameters: } A loss $\ell:\reals^k\times \Ycal\to [0,\infty)$, width parameter $\lambda_1$ and a regularization parameter $\lambda_2$
		\STATE \textbf{Input: } A dataset $(\bx_1,y_1),\ldots,(\bx_m,y_m) \in \left(\reals^d\right)^p\times \Ycal$
		\STATE Find symmetric $d\times d$ matrices $A_{i,j}, R$ for $1\le i\le k$ and $1\le j\le p$ by solving the semi-definite program:
\begin{eqnarray*} 
\min && \frac{1}{m}\sum_{i=1}^m \ell_{y_i}\left(\sum_{j=1}^p\bx_{i,j}^\top A_{1,j}\bx_{i,j},\ldots,\sum_{j=1}^p\bx_{i,j}^\top A_{k,j}\bx_{i,j}\right) 
+ \lambda_1\tr(R) + \lambda_2(\|R\|_\fr^2+\|\vec{A}\|_\fr^2 )
\\
s.t. &&  -R \preceq A_{i,j} \preceq R  \nonumber
\\
&& R \succeq 0 \nonumber
\end{eqnarray*}
\STATE Compute an  orthogonal diagonalization $R = U^\top DU$ 
\STATE Output $V = \sqrt{D}U$ and $P_{i,j} = (V^{\dagger})^\top A_{i,j} V^{\dagger}$.  
\end{algorithmic}
\end{algorithm}

\subsection{Conclusions and Future Work}
In this work we presented the novel RedEx architecture. This architecture is as expressive as neural networks, and can be trained in a layer-wise fashion using convex programs with semi-definite constraints. We also provided a separation result between RedEx and fixed representation methods based on a variation of the sparse-parity problem. Finally, we have shown several extensions of RedEx to the convolutional setting and replacing the semi-definite constraints to adding norm regularizers based on the newly introduced RedEx norm. Notably, for a one-dimensional input, it allows training of RedEx using non-constrained convex optimization algorithms such as gradient descent.


We believe our work can lead to more efficient representation learning methods based on convex optimization. This can include better and richer architectures, which may allow more efficient implementations that can be provably learned without the use of heavy convex SDP algorithms.
Finally, it is interesting to provide stronger separation results between RedEx and fixed representation methods under milder assumptions, e.g. in the case where the output is one-dimensional.


\bibliographystyle{abbrvnat}
\bibliography{bib}

\newpage
\appendix

\section{Proofs from \secref{Sec:RedEx}}\label{appen:proofs from sec RedEx}
\begin{proof}[Proof of \thmref{thm:RedEx epressive}]
    We first show that we can implement $\mathrm{AND}, \mathrm{OR}, \mathrm{NEG}$ and $\mathrm{Id}$ using a RedEx with $O(1)$ layers and feature dimension of $O(d^2)$. Recall that by definition, we added a coordinate to the data which is constant $1$. To implement $\mathrm{Id}$ of the $i$-th coordinate, we can use a matrix $V$ where the $i$-th row is equal to $\be_i$, and the last row (which corresponds to the constant $1$) equal to $\be_{d+1}$, this way  $\left((V\bx) (V\bx)^\top \right)_{i,d+1} = x_i$
    
    $\mathrm{NEG}$ can be implemented by $x\mapsto 1-x$. This can be implemented for coordinate $i$ by having the $i$-th row of $V$ equal to $-\be_i + \be_{d+1}$. We also need the last row of $V$ to be $\be_{d+1}$, this way $\left((V\bx) (V\bx)^\top \right)_{i,d+1} = 1-x_i$.
    
    Now, $\mathrm{AND}(x_i,x_j) = x_i\cdot x_j$ can be implemented by having $V$ with $i$-th row equal to $\be_i$, and $j$-th row equal to $\be_j$, this way  $\left((V\bx) (V\bx)^\top \right)_{i,j}=x_i\cdot x_j$. Finally, we have $\mathrm{OR}(x_1,x_2) = x_1 + x_2 - x_1\cdot x_2 $. This can be implemented by applying $\mathrm{Id}$ and $\mathrm{AND}$ on $x_1,x_2$, If the output of the above operations are in rows $i,j,k$ correspondingly, then we need some row of $V$ to be equal to $\be_i + \be_j - \be_k$ and the last row of $V$ to be equal to $\be_{d+1}$. 
    
    Note that if in the process of the quadratic expansion of RedEx we added extra coordinates which are not needed, in the next layer we can use zero rows for the unnecessary  coordinates to zero them out. This way, the application of the dimension reduction method would delete those unnecessary coordinates since their output is constant zero. Note that each operation above was implemented using at most $2$-layer RedEx, hence the feature dimension is at most $d^2$ where $d$ is the dimension of the input. For a general intermediate layer, we can bound its input by the total size of the target binary circuit, hence we can bound the feature dimension by $O(T^2)$. In addition, since each operation can be implemented by a RedEx of depth $O(1)$, the total depth of the RedEx which implements the Boolean circuit is $O(T)$.
\end{proof}

\section{Proofs from \secref{sec:learnability of RedEx}}\label{appen:proofs from learnability of RedEx}

\subsection{Proof of \thmref{thm:reduction learning RedEx}}
We first need the following lemma:

\begin{lemma}\label{lem:ker}
If $-R \preceq A \preceq R$ then $\ker(R)\subset \ker(A)$
\end{lemma}
\begin{proof}
    Since $A$ is symmetric and by the assumption of the lemma, it has a (non-unique) decomposition as $A=A_+ + A_-$ where $A_+$ is positive semi-definite with $A_+ \preceq R$ and $A_-$ is negative semi-definite with $-R \preceq A_-$.
    Let $0 \neq \bx\in\ker(R)$, then $0\leq \inner{\bx,A_+\bx}\leq \inner{\bx,R\bx} = 0$. Since $A_+$ is PSD it has an orthogonal diagonalization with orthonormal eigenvectors $\bv_i$ and corresponding eigenvalues $\lambda_i > 0$. We can exapnd $\bx$ in this basis $\bx = \sum_{i}\alpha_i\bv_i$. Now we have that:
    \begin{align*}
     0 &= \inner{\bx, A_+\bx} = \inner{\sum_{i}\alpha_i\bv_i, A_+ \sum_{i}\alpha_i\bv_i} = \inner{\sum_{i}\alpha_i\bv_i, \sum_{i}\lambda_i\alpha_i\bv_i} = \sum_{i}\alpha_i^2\lambda_i~.
    \end{align*}
    Hence, for every $i$ either $\alpha_i=0$ or $\lambda_i=0$, in particular, $\bx$ is in the kernel of $A_+$. Using a similar argument we get that $\bx$ is in the kernel of $A_-$, hence it is in the kernel of $A$.
    
\end{proof}

We are now ready to prove the main theorem:
\begin{proof}[\thmref{thm:reduction learning RedEx}]
    In the proof, for ease of notations we use the notion of RedEx norm, see Definition \ref{def:redex norm}.
    We begin with the first item. Let $h_{V,\vec{P}}\in\mathcal{H}_M$, we can write:
    \begin{align*}
        h_{V,\vec{P}}(\bx) &= \left(\tr\left(P_1 (V\bx)(V\bx)^\top\right),\dots,\tr\left(P_k (V\bx)(V\bx)^\top\right)\right) \\
        & = \left(\tr\left((V\bx)(V\bx)^\top P_1^\top\right),\dots,\tr\left( (V\bx)(V\bx)^\top P_k^\top\right)\right) \\
        & = \left(\tr\left((V\bx)(P_1V\bx)^\top \right),\dots,\tr\left( (V\bx)(P_kV\bx)^\top \right)\right) \\
        & = \left(\inner{V\bx, P_1 V\bx},\dots,\inner{V\bx, P_k V\bx}\right) \\
        & = \left(\inner{\bx, V^\top P_1 V\bx},\dots,\inner{\bx, V^\top P_k V\bx}\right)
    \end{align*}
It is therefore enough to show that for $A_i:=V^\top P_i V$ we have $\norm{\vec{A}}_\sd\leq M$. Since $\norm{P_1}_\text{sp}\leq 1$ we have $-I \preceq P_i \preceq I$. Hence, also $- V^\top V \preceq V^\top P_i V \preceq  V^\top V$. This implies that 
    \[
        \norm{\vec{A}}_\sd \leq  \tr(V^\top V) = \tr(V V^\top) \leq M~.
    \]

For the other direction, let $h^{R,\vec{A}}\in \mathcal{H}^M$ with $\norm{\vec{A}}_\sd \leq M$. Let $R\succeq 0$ be a matrix
which satisfies $-R\preceq A_i \preceq R$ for every $i\in[k]$, and $\tr(R)\leq M$. By diagonalizing $R$ we can write $R=U^\top D U$, where $U$ is unitary and $D$ is PSD. Define:
\begin{align*}
    V = \sqrt{D}U,~~ P_i = (V^\dagger)^\top A_i (V^\dagger)
\end{align*}
Here $\sqrt{D}$ is the diagonal matrix equal to $\sqrt{D}_{i,i} = \sqrt{D_{i,i}}$, and $V^\dagger$ is the pseudo-inverse of $V$. By its definition, $R$ has orthogonal rows. We also have that:
\begin{align*}
    \norm{P_i} &= \norm{(V^\dagger)^\top A_i (V^\dagger)} \\
    & = \norm{\left((\sqrt{D}U)^\dagger\right)^\top A_i(\sqrt{D}U)^\dagger} \\
    & \leq \norm{D^\dagger A_i} \leq 1
\end{align*}
where we used that $U$ is orthogonal and $A_i\preceq R$.
Now, we have that:
\begin{align*}
    \norm{V}^2_\fr = \tr (V^\top V) = \tr((\sqrt{D}U)^\top \sqrt{D}U) = \tr(D U^\top U) = \tr(D) = \tr(R) \leq M
\end{align*}
where we used that $U$ is orthogonal.
We have shown that for our definitions of $V$ and $\vec{P}$ we have that $h_{V,\vec{P}}\in \mathcal{H}_M$, it is left to show that for every $\bx\in\reals^d$ we have $h_{R,\vec{A}} (\bx) = h_{V,\vec{P}}(\bx)$.

Define $Q := V^\dagger V$, this is the projection on the range of $V^\top$, which contains the range of $R$. Hence, $I-Q$ is the projection on the orthogonal complement of the range of $V$, which is contained in the orthogonal complement of the range of $R$, which is the kernel of $R$. By \lemref{lem:ker} we have that $\ker(R) \subseteq \ker(A_i)$ for every $i
\in[k]$. Hence, $A_i(I-Q) = (I-Q)A_i = 0$, this implies that:
\begin{align*}
    A_i &= (I-Q + Q)A_i(I-Q + Q) \\
     & =(I-Q)A_i(I-Q) + (I-Q)A_iQ + QA_i(I-Q) + QA_iQ \\
     & = QA_i Q
\end{align*}
Finally, using that $P_i = (V^\dagger)^\top A_i V^\dagger$ we have for every $\bx\in\reals^d$:
\begin{align*}
    h_{V,\vec{P}}(\bx) &= \left(\inner{\bx, V^\top P_1 V\bx},\dots,\inner{\bx, V^\top P_k V\bx}\right) \\
    & = \left(\inner{\bx, V^\top(V^\dagger)^\top A_1 V^\dagger V\bx},\dots,\inner{\bx, V^\top (V^\dagger)^\top A_k V^\dagger  V\bx}\right) \\
    & = \left(\inner{\bx, Q^\top A_1 Q\bx},\dots,\inner{\bx, Q^\top A_k Q\bx}\right) \\
    & = \left(\inner{\bx,  A_1 \bx},\dots,\inner{\bx,  A_k \bx}\right) \\
    & = h_{R,\vec{A}}(\bx)
\end{align*}
This finishes the first part of the proof. For the second part of the theorem, we use the following algorithm to compute $V$ and $\vec{P}$ given $\vec{A}$:
\begin{enumerate}
\item
Find $R\succeq 0$ such that $\|\vec{A}\|_{\sd} = \tr(R)$ and $R\preceq A_i\preceq R$. 
\item  
Compute an orthogonal diagonalization $R = U^\top D U$. 
\item  
Output $V = \sqrt{D}U$ and $P_i = (V^{\dagger})^\top A_i V^{\dagger}$. 
\end{enumerate}
The first step can be completed in polynomial time using SDP solvers since this is a convex problem with linearly many constraints, see e.g. \cite{jiang2020faster}. The second step can also be done in polynomial time as it only consists of diagonlizing a symmetric matrix.
\end{proof}

\section{Proof of Theorem \ref{thm:main_layerwise}}\label{sec:main_thm_proof}
We first introduce some notation.
For $\bx\in\{\pm 1\}^d$ denote $\bx(\Ical) = \sum_{i\in \Ical}x_i$ and $P_i(\bx) = p_i(\bx(\Ical))$.
Denote $A_i^{t} = (V^{t})^\top P_i^{t}V^{t}$, $R^{t} = (V^{t})^\top V^{t}$, $\Psi_{t} = \Psi_{V^{t}}\circ\ldots\circ\Psi_{V^{(1)}}$ and
$h^{t}_i(\bx) = \Psi_{t}(\bx)^\top \vec{A}^{t}_i \Psi_{t}(\bx)$. Denote also $\Psi_0(\bx) = (c,\bx)\in\reals^{d+1}$ and will refer to the first coordinate in $\Psi_0(\bx)$ (the constant coordinate) as the $0$'th coordinate (instead of $1$'th). 
We note that $\Psi_{t}$ computes a polynomial of degree $\le 2^{t}$. 
Hence, $h^{t}_{i}$ is orthogonal to $p_{2i}\left(\bx(\Ical)\right)$ for $2i > 2^{t+1}$. Thus, for $i\ge 2^{t}$, the optimal solution to the $i$'th coordinate of the $t$'th layer is $0$. This observation motivated the definition of the $t$-truncated loss given by
\begin{equation*}
\val^t_{\Psi_{t-1}}(\vec{A},R) =    
\sum_{j=0}^{2^{t-1}}\E_{\bx} 
\left(p_{2j}(\bx(\Ical))-\Psi_{t-1}(\bx)^\top A_j\Psi_{t-1}(\bx)\right)^2 
\end{equation*}
We also denote
\[
\val^t_{\Psi_{t-1}} = \inf_{(\vec{A},R)}\val^t_{\Psi_{t-1}}(\vec{A},R)
\]
We denote by $P_{\Ical}:\reals^{d+1}\to\reals^{d+1}$ the projection on the coordinate in $\Ical$  and the first (constant) coordinate. That is, $(P_{\Ical}\bx)_{j} = \begin{cases}x_j & j\in\Ical\text{ or } j=0\\ 0 &\text{otherwise}\end{cases}$.

Before proceedeing to the main body of the proof, we  specialize theorem \ref{thm:generalization} for the square loss $\ell_y(\hat y) = \|\hat y -y\|^2$. While the square loss is not globally Lipschitz, it is Lipschitz on any bounded domain. Specifically, we have $\nabla_{\hat y}\ell_y(\hat y) = 2(\hat y - y) $. We also have that $\|\bx^\top\vec{A}\bx\|\le \sqrt{k}M\|\bx\|^2$. Hence, if $\Pr_{(x,y)\sim\Dcal}(\|y\|\le M_2) = 1$ then we have the the square loss is $(2M_2+2\sqrt{k}M_1^2M)$-Lipchitz in the relevant domain. Hence,
\begin{corollary}\label{cor:generalization_sq_loss}
Assume that the dataset is an i.i.d. sample from a distribution $\Dcal$ on $B^d_{M_1}\times B^k_{M_2}$ and that the loss is the square loss. Let $(\vec{A},R)$ be the output of algorithm \ref{alg:1-layer RedEx}. Then
\[
\E_S\val_{\Dcal}(\vec{A},R) \le \val_{\Dcal,M}+\lambda (k+1)M^2 + \frac{M_1^4(2M_2+2\sqrt{k}M_1^2M)^2}{\lambda m}
\]   
\end{corollary}
In our case, \eqref{eq:Kravchuk} implies that $|p_i(x)| \le k^i$. Hence, the output of the learned function is in $B^{1+k/2}_{k^{k}}$. Since $\lambda = \frac{1}{\sqrt{m}}$ we get
\[
\E_S\val_{\Dcal}(\vec{A},R) \le \val_{\Dcal,M} + \frac{(k+1)M^2 + M_1^4(2k^k+2\sqrt{k}M_1^2M)^2}{\sqrt{m}}
\]   

\subsection{First layer}\label{sec:first_layer}
Let $1_{\Ical}, 1_0\in \reals^{d+1}$ be the indicator vectors of $\Ical$ and $\{0\}$ .
Denote by $J, I_{\Ical}$ and $I_0$ the $(d+1)\times (d+1)$ matrices given by
\[
J = \frac{1}{k}1_\Ical 1^\top_\Ical,\;\; I_0=1_0 1^\top_0,\;\; I_\Ical = \frac{1}{k}(P_{\Ical} - I_0)
\]
We note that
\[
A_0 = \frac{1}{2}I_0,\;\;\;A_1 = \frac{1}{\sqrt{2-2/k}}J - \frac{1}{2\sqrt{2-2/k}}I_0,\;\; R =  \frac{1}{2\sqrt{2-2/k}}J + \frac{1}{2}I_0
\]
Is a solution to the first layer with zero $1$-truncated loss. Thus, by corollary \ref{cor:generalization_sq_loss} we will have $\val^1_{\Psi_{0}}(\vec{A}^1,R^1) \le \frac{2(k+1) + (d+1)^2(2k^k+4\sqrt{k}(d+1))^2}{\delta\sqrt{m}}=:\epsilon_1$ w.p. $1-\delta$.
The following lemma shows that in this case it holds that $\|V^{1}(I-P_\Ical)\|_\spec \le (8d\epsilon_1)^\frac{1}{4} =: \epsilon$.

\begin{lemma}\label{lem:optimal_first_layer_new}
If $\val^1_{\Psi_0}(\vec{A}^1,R^1) \le \epsilon^2$ then $\|V^1(I-P_{\Ical})\|_\spec \le (8d)^{1/4}\sqrt{\epsilon}$
\end{lemma}
\begin{proof}
We have that $\|h^1_0-P_0\|^2_2 + \|h^1_1-P_2\|^2_2\le\epsilon^2$. Hence, there are $(d+1)\times(d+1)$ matrices $A'_0,A'_1$ with $\|A^1_0-A_0'\|_F^2 +  \|A^1_1-A_1'\|_F^2\le \epsilon^2$ such that $P_0(\bx) = \Psi_0(\bx)^\top A_0'\Psi_0(\bx)$ and $P_2(\bx) = \Psi_0(\bx)^\top A_1'\Psi_0(\bx)$. Since $\|A_i'-A^1_i\|_\tr\le \sqrt{d}\|A_i'-A^1_i\|_F\le \sqrt{d}\epsilon$, there is a PSD matrix $R'$ such that $-R'\preceq A_i'\preceq R'$ and $\tr(R')\le \tr(R^1) + \sqrt{2d}\epsilon \le \frac{1}{2} + \frac{1}{\sqrt{2-2/k}} + \sqrt{2d}\epsilon$.

Now, consider the matrices $A''_i, R''$ obtained by zeroing $(A'_i)_{jj}$, $R'_{jj}, R'_{jl}$ and $R'_{lj}$ for any $j\in [d]\setminus \Ical$ and $0\le l\le d$, and adding $\frac{1}{2}\sum_{j\in [d]\setminus I}(A'_i)_{jj}$ to $(A'_i)_{00}$ as well as $\frac{1}{2}\sum_{j\in [d]\setminus I}R'_{jj}$ to $R'_{00}$. 
We have $\Psi_0(\bx)^\top A_i''\Psi_0(\bx) = \Psi_0(\bx)^\top A_i'\Psi_0(\bx)$, $-R''\preceq A_i''\preceq R''$ and $\tr(R'') = \tr(R') - \frac{1}{2}\sum_{j\in [d]\setminus I}R'_{jj}$. 
lemma \ref{lem:opt_tr_norm_of_first_layer}
now implies that
\[
\frac{1}{2}+ \frac{1}{\sqrt{2-2/k}}  \le \tr(R'') = \tr(R') - \frac{1}{2}\sum_{j\in [d]\setminus I}R'_{jj} \le  \frac{1}{2}+ \frac{1}{\sqrt{2-2/k}} + \sqrt{2d}\epsilon - \frac{1}{2}\sum_{j\in [d]\setminus I}R'_{jj}
\]
Hence,
\[
\sum_{j\in [d]\setminus I}R^1_{jj} \le \sum_{j\in [d]\setminus I}R'_{jj} \le \sqrt{8d}\epsilon
\]
Finally, we have
\begin{eqnarray*}
\|V^1(I-P_{\Ical})\|_\spec^2 &=& \|(I-P_{\Ical})(V^1)^\top V^1(I-P_{\Ical})\|_\spec
\\
&=& \|(I-P_{\Ical})R^1(I-P_{\Ical})\|_\spec
\\
&\le& \|(I-P_{\Ical})R^1(I-P_{\Ical})\|_\tr
\\
&=& \sum_{j\in [d]\setminus \Ical}R^1_{jj}
\\
&\le& \sqrt{8d}\epsilon    
\end{eqnarray*} 
\end{proof}

\begin{lemma}\label{lem:opt_tr_norm_of_first_layer}
    Let $A_0,A_1, R$ be $(d+1)\times (d+1)$ matrices such that $-R\preceq A_i\preceq R$ and $P_{2i}(\bx) = (\Psi_0(\bx))^\top A_i\Psi_0(\bx)$. Then, $\tr(R)\ge \frac{1}{2} + \frac{1}{\sqrt{2+2/k}}$
\end{lemma}
\begin{proof}
    We assume w.l.o.g. that $A_0,A_1, R$ minimizes $\tr(R)$ under the above constraints. It is not hard to verify that $A_0$ and $A_1$ are linear combination of the PSD matrices $J, I_0, \tilde{I}_\Ical := \frac{k}{k-1}(I_\Ical - \frac{1}{k}J)$ and $I_{\Ical^c} := \frac{1}{d-k}(I -  P_\Ical)$. 
    Write 
    \[
    A_i = a^1_iJ +a^2_iI_0 + a^3_i\tilde{I}_\Ical + a^4_i I_{\Ical^c}
    \]
    Since the the matrices $J, I_0, \tilde{I}_\Ical$ and $I_{\Ical^c}$ are supported on orthogonal spaces, the minimal trace of a PSD matrix $R$ with $-R\preceq A_i \preceq R$ is $\sum_{j=1}^4\max(|a^j_0|,|a^j_1|)$.
    Now, zeroing $a_0^4$ and $a_1^4$ while adding $\frac{1}{2}a_0^4$ and $\frac{1}{2}a_1^4$ to $a_0^1$ and $a_1^1$ will not alter the functions computed by $A_0$ and $A_1$ and will not increase $\sum_{j=1}^4\max(|a^j_0|,|a^j_1|)$. Thus, we can assume that $a_0^4=a_1^4=0$.
    
    Likewise, zeroing $a_0^3$ and $a_1^3$ while adding $-\frac{a_i^3}{k-1}$ to $a_i^1$, and $\frac{k}{k-1}\frac{a_i^3}{2}$ to $a_i^2$ will not alter the functions computed by $A_0$ and $A_1$ and will not increase $\sum_{j=1}^4\max(|a^j_0|,|a^j_1|)$. Thus, we can assume that $a_0^3=a_1^3=0$. This implies that $A_0 = \frac{1}{2}I_0$ and $A_1 = \frac{1}{\sqrt{2-2/k}}J - \frac{1}{2\sqrt{2-2/k}}I_0$
\end{proof}

\subsection{Remaining Layers and conclusion of the proof}\label{sec:remaining_layer}
Suppose that the $t$'th layer has $t$-truncated error at most $\epsilon_t$ and that $\|V^1(I-P_\Ical)\|_\spec \le \epsilon$. Lemma \ref{lem:boundness_properties} below implies that there is a solution for the $t$'th layer with Frobenius norm at most $2^{k/2}$ and 
$t$-truncated error at most $\epsilon_t + dk^{O(k^2)}\epsilon^2$. Lemma \ref{lem:next_layer_sol} below now implies that there is a solution for the $(t+1)$'th layer with 
$(t+1)$-truncated error at most $k^{O(k^2)}\epsilon_t + dk^{O(k^2)}\epsilon^2$. By lemma \ref{lem:boundness_properties} and corollary \ref{cor:generalization_sq_loss} we have that w.p. $1-\delta$, the $(t+1)$'th layer has $(t+1)$-truncated error at most $k^{O(k^2)}\epsilon_t + dk^{O(k^2)}(\epsilon^2 + 1/\delta\sqrt{m})$. By induction, we conclude that w.p. $1-t\delta$, the truncated error of the $t$'th layer is
$k^{t O(k^2)}\epsilon_1 + dk^{tO(k^2)}(\epsilon^2 + 1/\delta\sqrt{m})$.
Hence, w.p. $1-\log(k)\delta$, the error of the final layer is
\[
k^{ O(k^2\log(k))}\epsilon_1 + dk^{O(k^2\log(k))}(\epsilon^2 + 1/\delta\sqrt{m}) = \frac{d^42^{O(k^2\log^2(k))}}{\delta m^{1/4}}
\]
which concludes the proof.

\begin{lemma}\label{lem:boundness_properties}
If $\|V^1(I-P_\Ical)\|_\spec \le \epsilon$ then, for any $\bx\in [-1,1]^d$ we have, 
\begin{enumerate}
    \item $\|\Psi_t(\bx)\|_F\le   M^{2^t-2}(M|\Ical| + \epsilon d)^{2^{t}}  \le M^{3\cdot 2^t}$
    \item $\|\Psi_t(\bx) - \Psi_t(P_\Ical\bx)\|_F\le   d 2^{t} M^{2t + 3\cdot (2^{t}-1)}\epsilon $
    \item For the projection $P$ on $\spn\left( V^t\Psi_{t-1}(P_\Ical \{\pm 1\}^d )\right)$ and any symmetric $A$ with $\|A\|_\spec\le 1$ and $\bx\in \{\pm 1\}^d$ we have
\[
\left|\inner{A,\Psi_t(\bx)} - \inner{P A P,\Psi_t(\bx)}\right| \le  d 2^{t} M^{2t + 3\cdot (2^{t}-1)}\epsilon \text{ and }\|P A P\|_F\le 2^{|\Ical|/2}
\]

\end{enumerate}
\end{lemma}
\begin{proof}
We first prove item 1. by induction on $t$. For $t=1$ we have
\begin{eqnarray*}
\|\Psi_1(\bx)\|_F &=& \|V^1\bx\|^2 
\\
&\le& \left(\|V^1(I-P_\Ical)\|_\spec\|\bx\| + \|V^1\|_\spec\|P_\Ical \bx\|\right)^2
\\
&\le&  \left(\epsilon \sqrt{d} + M \sqrt{|\Ical|}\right)^2
\end{eqnarray*}
For $t>1$ we have by the induction hypothesis
\begin{eqnarray*}
\|\Psi_t(\bx)\|_F &=& \|V^t \Psi_{t-1}(\bx)\|^2
\\
&\le& M^2 \left(M^{2^{t-1}-2}(M|\Ical| + \epsilon d)^{2^{t-1}}\right)^2
\\
&=& M^{2^{t}-2}(M|\Ical| + \epsilon d)^{2^{t}}
\end{eqnarray*}
We next prove item 2. by induction on $t$. For $t=1$ we have
\begin{eqnarray*}
\|\Psi_1(\bx) - \Psi_1(P_\Ical\bx)\|_F &=& \|(V^1\bx) \otimes (V^1\bx) - (V^1P_\Ical\bx)\otimes (V^1P_\Ical\bx)\| 
\\
&\le & \|(V^1\bx)\otimes (V^1(I-P_\Ical)\bx)\| +\| (V^1(I-P_\Ical)\bx)\otimes (V^1P_\Ical\bx)\| 
\\
&\le& 2dM\epsilon
\end{eqnarray*}
For $t>1$ we have by the induction hypothesis and item 1.
\begin{eqnarray*}
\|\Psi_t(\bx) - \Psi_t(P_\Ical\bx)\|_F &=& \|(V^t\Psi_{t-1}(\bx)) \otimes (V^t\Psi_{t-1}(\bx)) - (V^t \Psi_{t-1}(P_\Ical\bx))\otimes (V^t \Psi_{t-1}(P_\Ical\bx))\| 
\\
&\le & \|(V^t\Psi_{t-1}(\bx)) \otimes (V^t\Psi_{t-1}(\bx) - V^t\Psi_{t-1}(P_\Ical\bx))\| \\
& & + \|(V^t\Psi_{t-1}(\bx) - V^t\Psi_{t-1}(P_\Ical\bx))\otimes(V^t\Psi_{t-1}(P_\Ical\bx)) \| 
\\
&\le& M^2\cdot \|\Psi_{t-1}(\bx))\|\cdot \|\Psi_{t-1}(\bx) - \Psi_{t-1}(P_\Ical\bx)) \|
\\
&& + M^2\cdot \|\Psi_{t-1}(P_\Ical\bx))\|\cdot \|\Psi_{t-1}(\bx) - \Psi_{t-1}(P_\Ical\bx)) \|
\\
&\le & 2M^2 M^{3\cdot 2^{t-1}} 2^{t-1} d M^{2(t-1) + 3\cdot (2^{t-1}-1)}\epsilon 
\\
&= &  2^{t} d M^{2t + 3\cdot (2^{t}-1)}\epsilon 
\end{eqnarray*}
We now prove item 3.
\begin{eqnarray*}
    \inner{A,\Psi_t(\bx)} - \inner{P A P,\Psi_t(\bx)} &=& (V^t\Psi_{t-1}(\bx))^\top A (V^t\Psi_{t-1}(\bx))
    \\
    & & -( PV^t\Psi_{t-1}(\bx))^\top A (P V^t\Psi_{t-1}(\bx))
    \\
    &=& (V^t\Psi_{t-1}(\bx))^\top A (V^t\Psi_{t-1}(\bx) - PV^t\Psi_{t-1}(\bx))
    \\
    & & +(V^t\Psi_{t-1}(\bx) - PV^t\Psi_{t-1}(\bx))^\top A (P V^t\Psi_{t-1}(\bx))
\end{eqnarray*}
Hence, by the previous items,
\begin{eqnarray*}
    \left|\inner{A,\Psi_t(\bx)} - \inner{P A P,\Psi_t(\bx)}\right\| &\le & \|A\|\spec \cdot  \|V^t\Psi_{t-1}(\bx) - PV^t\Psi_{t-1}(\bx)\| \cdot (\|V^t\Psi_{t-1}(\bx)\| + \|P V^t\Psi_{t-1}(\bx)\|)
    \\
    &\le & 2 M^{1+3\cdot 2^{t-1}}\|V^t\Psi_{t-1}(\bx) - PV^t\Psi_{t-1}(\bx)\|
    \\
    &\le & 2 M^{1+3\cdot 2^{t-1}}\|V^t\Psi_{t-1}(\bx) - V^t\Psi_{t-1}(P_\Ical\bx)\|
    \\
    &\le & 2 M^{2+3\cdot 2^{t-1}}d 2^{t-1} M^{2(t-1) + 3\cdot (2^{t-1}-1)}\epsilon 
    \\
    &\le & d 2^{t} M^{2t + 3\cdot (2^{t}-1)}\epsilon 
\end{eqnarray*}
Finally,
\[
\|P A P\|_F\le  \|P A P\|_\spec \sqrt{\mathrm{rank}(PAP)} \le  1\cdot \sqrt{\mathrm{rank}(P)}\le  2^{|\Ical|/2}
\]

\end{proof}

\begin{lemma}
For even $i\ge 0$ denote
\[
T_i = \left\{ (j,l) : 0\le j\le l\le \max(2,2^{\lceil \log_2(i/2)\rceil}) \text{ and }j+l\le i\text{ and }j,l\text{ are even}\right\}
\]
for odd $i\ge 0$ denote
\[
T_i = \left\{ (j,l) : 0\le j, l\le \max(2,2^{\lceil \log_2(i/2)\rceil}) \text{ and }j+l\le i\text{ and }j\text{ is odd and }l\text{ is even}\right\}
\]
There are coefficients $\left\{\alpha^{i}_{j,l}\right\}_{(j,l\in T_i)}$ such that 
\[
p_i = \sum_{(j,l)\in T_i} \alpha^{i}_{j,l}p_jp_l
\]
furthermore, $|\alpha^{i}_{j,l}| \le (3k)^i$
\end{lemma}
\begin{proof}
By induction on $i$. For $i=0$ we have $p_0 = p_0 p_0$ and of $i=1$ we have $p_1 = p_0 p_1$. For any $i + 1\ge 1$  we have by  \eqref{eq:Kravchuk}
\begin{eqnarray*}
p_{i+1}(x) = \frac{1}{\sqrt{(i+1)(k-i)}}xp_i(x) - \frac{\sqrt{i(k-i+1)}}{\sqrt{(i+1)(k-i)}}p_{i-1}(x)
\end{eqnarray*}
By the induction hypothesis and \eqref{eq:Kravchuk}
\begin{eqnarray*}
p_{i+1}(x) & = & \frac{1}{\sqrt{(i+1)(k-i)}}\sum_{(j,l)\in T_i}\alpha^i_{j,l}xp_j(x)p_l(x) + \frac{\sqrt{i(k-i+1)}}{\sqrt{(i+1)(k-i)}}\sum_{(j,l)\in T_{i-1}}\alpha^{i-1}_{j,l}p_j(x)p_l(x)
\\
 & = & \sum_{(j,l)\in T_i}\alpha^i_{j,l}\left(\frac{\sqrt{(j+1)(k-j)}}{\sqrt{(i+1)(k-i)}}p_{j+1}(x) + \frac{\sqrt{j(k-j+1)}}{\sqrt{(i+1)(k-i)}}p_{j-1}(x)\right)p_l(x)
\\
 & & + \frac{\sqrt{i(k-i+1)}}{\sqrt{(i+1)(k-i)}}\sum_{(j,l)\in T_{i-1}}\alpha^{i-1}_{j,l}p_j(x)p_l(x)
\end{eqnarray*}
The lemma follows from the fact that
\[
\frac{\sqrt{(j+1)(k-j)}}{\sqrt{(i+1)(k-i)}}
+\frac{\sqrt{j(k-j+1)}}{\sqrt{(i+1)(k-i)}}
+\frac{\sqrt{i(k-i+1)}}{\sqrt{(i+1)(k-i)}} \le 3\sqrt{k} \le 3k
\]
\end{proof}

\begin{lemma}\label{lem:next_layer_sol}
Let $\Psi:\{\pm 1\}^d\to B_{M_1}^n$. Assume that for any $0\le i\le 2^{t-1}$ there is a vector $\bv_i\in B_{M_2}^n$ such that for $\tilde P_i(\bx):=\inner{\bv_i,\Psi(\bx)}$ we have $\left\|P_i - \tilde P_i\right\|_2\le \epsilon$. Then, there are matrices $A_i$ for $0\le i\le 2^t$ such that for $\hat P_i(\bx) := \Psi(\bx)^\top A_i\Psi(x)$ we have $\left\|P_i - \hat P_i\right\|_2\le 2 M_1M_2 (3|\Ical|)^{|\Ical|}|\Ical|^2 \epsilon$ and $\|A_i\|_\tr \le 2^{t}(3k)^{2^t}M_2^2$
\end{lemma}
\begin{proof}
Fix $0\le i \le 2^t$ and consider the matrix
\[
A_i = \sum_{j,l\in T_i}\alpha^i_{j,l}\bv_j\bv_l^\top
\]
We have $\|A_i\|_\tr\le 2^{t}(3k)^{2^t}M_2^2$ and
\[
\hat{P_i}(\bx) = \Psi(\bx)^\top A_i\Psi(x)= \sum_{j,l\in T_i}\alpha^i_{j,l}\tilde{P_j}(\bx) \tilde{P_l}(\bx)
\]
Hence,    
\begin{eqnarray*}
 \left\|P_i - \hat P_i\right\|_2 &=& \left\|\sum_{j,l\in T_i}\alpha^i_{j,l}\left(P_jP_l - \tilde P_j\tilde P_l\right)\right\|_2
\\
&\le& (3|\Ical|)^{|\Ical|}\sum_{j,l\in T_i} \left\|P_jP_l -   \tilde P_j\tilde P_l\right\|_2
\\
&\le& (3|\Ical|)^{|\Ical|}\sum_{j,l\in T_i} \left\|P_jP_l -    P_j\tilde P_l\right\|_2 + \left\|P_j\tilde P_l -   \tilde P_j\tilde P_l\right\|_2
\\
&\le& (3|\Ical|)^{|\Ical|} \sum_{j,l\in T_i} \|P_j\|_\infty\left\|P_l -    \tilde P_l\right\|_2 + \|\tilde P_l\|_\infty\left\|P_j -   \tilde P_j\right\|_2
\\
&\le& 2 M_1M_2 (3|\Ical|)^{|\Ical|}|\Ical|^2 \epsilon
\end{eqnarray*}

\end{proof}

\section{Proofs and Additional Algorithms from \secref{sec:discussion}}\label{appen:proofs from discussion}

\subsection{Proof of \lemref{lem:properties of sd norm}}

It is clear that $\|\cdot \|_{\sd}$ is homogeneous and non-negative. It remains to show that the triangle inequality is satisfied and that $\|\vec{A}\|_{\sd} > 0$ for $\vec{A}\ne 0$. For the triangle inequality we have:
\begin{eqnarray*}
\| \vec{A} + \vec{B} \|_{\sd} &=& \min\{ \tr(R) :R\succeq 0\text{ and } \forall i,-R \preceq B_i + A_i \preceq R  \}
\\
&\le & \min\{ \tr(R_1) + \tr(R_2) :R_1,R_2\succeq 0\text{ and } \forall i,-R_2 \preceq A_i \preceq R_i\text{ and }-R_2 \preceq B_i \preceq R_2  \}
\\
&=& \| \vec{A} \|_{\sd} + \| \vec{B} \|_{\sd}
\end{eqnarray*}
Let $\vec{A}\ne 0$, then there is $j\in[k]$ with $A_j\neq 0$ which also means that $\norm{A_j}_{\tr} > 0$. Note that 
\[
\| \vec{A}  \|_{\sd} \ge  \max_{i\in [k]}  \| A_i \|_{\tr}  \geq \norm{A_j}_{\tr} > 0
\]

For the second part, we need to find $R$ which minimizes:
\begin{equation*}
    \min\tr(R) \text{ ~~s.t. }~-R\preceq A \preceq R,~~  0 \preceq R~.
\end{equation*}
First, assume that $A$ is a diagonal matrix. For every unit vector $\be_i$ we have that $-\be_i^\top R \be_i \leq \be_i A\be_i \leq \be_i R\be_i$, which means that $-r_{i,i} \leq a_{i,i} \leq r_{i,i} $. In other words, we get that $r_{i,i} \geq |a_{i,i}|$, and the minimum on the trace of $R$ is achieved when $r_{i,i} = a_{i,i}$ for every $i$. Consider some $R$ that achieves the minimum, and assume it is not a diagonal matrix. Then, there are indices $i\neq j$ with $r_{i,j}\neq 0$ (and also $r_{i,j}\neq 0$ since $R$ is symmetric). Assume that $r_{i,j} > 0$ and let $\bv$ be the vector with $1$ in the $i$-th and $j$-th coordinates and $0$ in every other coordinate. By the condition of $R$ we get that $\bv^\top (R-A) \bv \geq 0$. But we have that:
\begin{align*}
    \bv^\top (R-A) \bv  = -2r_{i,j} <0
\end{align*} 
which is a contradiction. In case $r_{i,j} <0$ we can take $\bv$ to be equal $1$ in the $i$-th coordinate, $-1$ in the $j$-th coordinate and $0$ in every other coordinate. This shows that if $A$ is diagonal, then there is a single solution for the minimization problem with a diagonal $R$ such that $r_{i,i} = |a_{i,i}|$ for every $i$. 

Assume now that $A$ is some symmetric matrix, and write $A = U^\top D U$ where $U$ is orthogonal and $D$ is diagonal. Let $R$ be some solution to the minimization problem, then $U^\top R U$ is a solution to the same minimization problem where we replace $A$ with $D$. But since $D$ is diagonal, then there is a single solution where $U^\top R U$ is diagonal with $(U^\top R U) = |D|$. Hence $R = U^\top |D| U$ is the single solution to the minimization problem for $A$. Finally, we have:
\[
\norm{A}_\tr = \sum_i |D_{i,i}| = \tr(R) = \norm{A}_\sd~.
\]

\subsection{Additional Algorithms}
In \secref{sec:discussion} we provided several additional algorithms for training RedEx using norm constraints instead of semi-definite constrains. Here we provide the full algorithms. In Algorithm \ref{alg:1-layer RedEx using norm} we provide the 1-layer RedEx algorithm where we use the RedEx norm instead of the semi-definite constraints on $R$. Note that the optimization algorithm does not include $R$, although it does require finding $R$ to output $V$. In Algorithm \ref{alg:1-layer RedEx output dim 1} we show how to train a 1-layer RedEx where $k=1$. In this case, the RedEx norm is equivalent to the trace norm, hence this algorithm requires solving an unconstrained optimization problem. This can be solved using standard gradient methods such as GD or SGD. Note that it is also not needed to explicitly find $R$, since by \lemref{lem:properties of sd norm} it can be calculated directly from $A$.

\begin{algorithm}[H]
	\caption{Training 1 -layer RedEx with the RedEx norm}
	\begin{algorithmic}[1]\label{alg:1-layer RedEx using norm}
		\STATE \textbf{Parameters: } A loss $\ell:\reals\times \Ycal\to [0,\infty)$, width parameter $\lambda_1$ and a regularization parameter $\lambda_2$
		\STATE \textbf{Input: } A dataset $(\bx_1,y_1),\ldots,(\bx_m,y_m) \in \reals^d\times \Ycal$
		\STATE Find symmetric $d\times d$ matrices $A_1,\ldots,A_k$ by solving the program:
\begin{eqnarray} 
\min && \frac{1}{m}\sum_{i=1}^m \ell_{y_i}\left(\bx_i^\top \vec{A}\bx_i\right) 
+ \lambda_1\|\vec{A}\|_\sd + \lambda_2\|\vec{A}\|_\fr^2
\end{eqnarray}
\STATE Find $R$ that minimizes \eqref{eq:sd norm def}
\STATE Compute an  orthogonal diagonalization $R = U^\top DU$ 
\STATE Output $V = \sqrt{D}U$ and $P_i = (V^{\dagger})^\top A_i V^{\dagger}$.  
\end{algorithmic}
\end{algorithm}

\begin{algorithm}[H]
	\caption{Training 1 -layer RedEx with output dimension $1$}
	\begin{algorithmic}[1]\label{alg:1-layer RedEx output dim 1}
	\STATE \textbf{Parameters: } A loss $\ell:\reals\times \Ycal\to [0,\infty)$, width parameter $\lambda_1$ and a regularization parameter $\lambda_2$
	\STATE \textbf{Input: } A dataset $(\bx_1,y_1),\ldots,(\bx_m,y_m) \in \reals^d\times \Ycal$
	\STATE Find symmetric $d\times d$ matrix $A$ by solving the program:
\begin{eqnarray} 
\min && \frac{1}{m}\sum_{i=1}^m \ell_{y_i}\left(\bx_i^\top A\bx_i\right) 
+ \lambda_1\|A\|_\tr + 2\lambda_2\|A\|_\fr^2
\end{eqnarray}
\STATE Compute an  orthogonal diagonalization $A = U^\top DU$ 
\STATE Output $V = \sqrt{|D|}U$ and $P = (V^{\dagger})^\top A V^{\dagger}$.  
\end{algorithmic}
\end{algorithm}


\end{document}

%% file: math_commands.tex
\usepackage{amsmath,amsfonts,bm}

\def\secref#1{section~\ref{#1}}

\def\eqref#1{equation~\ref{#1}}

\def\algref#1{algorithm~\ref{#1}}

\def\1{\bm{1}}

\DeclareMathAlphabet{\mathsfit}{\encodingdefault}{\sfdefault}{m}{sl}
\SetMathAlphabet{\mathsfit}{bold}{\encodingdefault}{\sfdefault}{bx}{n}

\newcommand{\E}{\mathbb{E}}

